\documentclass[a4wide,11pt]{article}

\usepackage{fullpage}
\usepackage{lineno,hyperref}
\modulolinenumbers[5]

\usepackage[font=footnotesize]{subfig, caption}
\usepackage{graphicx}
  \graphicspath{{.}}
  \DeclareGraphicsExtensions{.pdf,.jpeg,.png,.jpg}

\usepackage{xcolor}

\usepackage{outlines}
\usepackage[nolist]{acronym}
\usepackage{algorithm}

\usepackage{algpseudocode}
\algdef{SE}[SUBALG]{Indent}{EndIndent}{}{\algorithmicend\ }%
\algtext*{Indent}
\algtext*{EndIndent}

\usepackage{colortbl}
\usepackage{listings}

\usepackage{amsmath}
\usepackage{amsthm}
\usepackage{amssymb}
\theoremstyle{definition}
\newtheorem{definition}{Definition}[section]
\newtheorem{theorem}{Theorem}[section]

\newtheorem{remark}{Remark}[section]
\theoremstyle{plain}
\newtheorem{example}{Example}[section]

\usepackage[inline]{enumitem}

\usepackage{xfp}

\usepackage{booktabs}
\usepackage{multirow}
\usepackage{lmodern}
\usepackage[T1]{fontenc}

\usepackage{xspace}
\newcommand{\nd}{\texttt{NegDis}\xspace}
\newcommand{\declare}{\texttt{Declare}\xspace}
\newcommand{\rum}{\texttt{RuM}\xspace}
\newcommand{\declareminer}{\texttt{DeclareMiner}\xspace}
\newcommand{\decminer}{\texttt{DecMiner}\xspace}
\newcommand{\minclos}{\textsc{Cardinality}\xspace}
\newcommand{\subsetclos}{\textsc{Subset}\xspace}

\newcommand\para[1]{\textnormal{\textsf{#1}}}

\newcommand\paragrafo[1]{{\smallskip \noindent \textit{#1}.}}

\newcommand{\sheriff}{sheriffs}

\begin{document}

\title{Process discovery on deviant traces and \\ other stranger things}

\author{Federico~Chesani$^1$, Chiara~Di~Francescomarino$^2$, Chiara~Ghidini$^2$, Daniela~Loreti$^1$, \\
Fabrizio~Maria~Maggi$^3$, Paola~Mello$^1$, Marco~Montali$^3$, Sergio~Tessaris$^3$\\
$^1$ DISI - University of Bologna, Bologna, Italy\\ 
$^2$ Fondazione Bruno Kessler (FBK), Trento, Italy \\
$^3$ Free University of Bozen-Bolzano, Bolzano, Italy \\
E-mail: \url{daniela.loreti@unibo.it}
}

\date{}

\maketitle

\begin{abstract}
As the need to understand and formalise business processes into a model has grown over the last years, 
the process discovery research field has gained more and more importance, developing two different classes of approaches to model representation: procedural and declarative. 
Orthogonally to this classification, the vast majority of works envisage the discovery task as a one-class supervised learning process guided by the traces that are recorded into an input log. 

In this work instead, we focus on declarative processes and embrace the less-popular view of process discovery as a binary supervised learning task, where the input log reports both examples of the normal system execution, and traces representing ``stranger'' behaviours according to the domain semantics. We therefore deepen how the valuable information brought by both these two sets can be extracted and formalised into a model that is ``optimal'' according to user-defined goals. Our approach, namely \nd, is evaluated w.r.t. other relevant works in this field, and shows promising results as regards both the performance and the quality of the obtained solution.
\end{abstract}


\section{Introduction}\label{sec:intro}
The modelling of business processes is an important task to support decision-making in complex industrial and corporate domains. Recent years have seen the birth of the \ac{BPM} research area, focused on the analysis and control of process execution quality, and in particular, the rise in popularity of \emph{process mining} \cite{2012-Aalst}, which encompasses a set of techniques to extract valuable information from event logs. 
\emph{Process discovery} is one of the most investigated process mining techniques. It deals with the automatic learning of a process model from a given set of logged traces, each one representing the digital footprint of the execution of a case. 
Process discovery algorithms are usually classified into two categories according to the language they employ to represent the output model: procedural and declarative.
Procedural techniques envisage the process model as a synthetic description of all possible sequences of actions that the process accepts from an initial to an ending state. Declarative discovery algorithms---which represent the context of this work---return the model as a set of constraints equipped with a declarative, logic-based semantics, and that must be fulfilled by the traces at hand. 
Both approaches have their strengths and weaknesses depending on the characteristics of the considered process. For example, procedural techniques often produce  intuitive models, but may sometimes lead to ``spaghetti''-like outputs \cite{2009-Fahland, 2018b-Maggi}: in these cases declarative-based approaches might be preferable.

Declarative techniques rely on shared metrics to establish the quality of the extracted model, for example in terms of \emph{fitness}, \emph{precision}, \emph{generality}, and \emph{simplicity} \cite{2015-Adriansyah,2014-Broucke,2018-Ponce}. In particular, fitness and precision focus on the quality of the model w.r.t.~the log, i.e., its ability to accept desired traces and reject unlikely ones, respectively; generality measures the model's capability to abstract from the input by reproducing the desired behaviours, which are not assumed to be part of the log in the first place; finally, simplicity is connected to the clarity and understandability of the result for the final user. 

Besides the declarative-procedural classification, process discovery approaches can be also divided into two categories according to their vision on the model-extraction task. 
As also pointed out by Ponce-de-Le\`on et al. \cite{2018-Ponce}, the vast majority of works in the process discovery spectrum (e.g. \cite{2004-Aalst,2003-Weijters,2007-Gunther,2010-Aalst}) can be seen as one-class supervised learning technique, while fewer works (e.g. \cite{2006-Maruster,2009-Goedertier,2009-Chesani}) intend model-extraction as a two-class supervised task---which is driven by the possibility of partitioning the log traces into two sets according to some business or domain-related criterion. Usually these sets are referred to as \emph{positive} and \emph{negative} examples \cite{2018-Ponce}, and the goal is to learn a model that characterises one set w.r.t. the other.

A further consideration stems from the \emph{completeness} of the log. Generally, a log contains and represents only a subset of the possible process executions. Other executions might be accepted or rejected from the viewpoint of the process, but this can be known only when a process model is learned or made available. This territory of \emph{unknown traces} will be ``shaped'' by the learned model, and more precisely by the choice of the discovery technique, and possibly by its configuration parameters. Approaches that consider positive examples only provide a number of heuristics, thus allowing the user to decide to which extent the yet-to-be-seen traces will be accepted or rejected by the discovered model---ranging from the extremity of accepting them all, to the opposite of rejecting them all.
The use of a second class of examples, identified on the basis of some domain-related criterion, allows to introduce some business-related criterion besides the heuristics.

In this work, we focus on declarative process models expressed in the Declare language \cite{2008-Pesic}, and embrace the view of process discovery as a binary supervised learning task. Hence, our starting point is a classification of business traces into two sets, which can be driven by various motivations.
For example, in order to avoid underfitting and overfitting \cite{2010-Aalst}, many authors, as well as the majority of tools, suggest to ignore less-frequent traces (sometimes referred as \emph{deviances} from the usual behaviour \cite{2016-Nguyen}), thus implicitly splitting the log according to a frequency criterion. Another motivation for log partitioning could be related to the domain-specific need to identify ``stranger'' execution traces, e.g., traces asking for more (or less) time than expected to terminate, or undesirable traces that the user might want to avoid in future executions.


%
Independently of the chosen criteria for splitting the log, we adopt the terms negative and positive example sets to identify the resulting partitioning, keeping in mind that the ``negative'' adjective is not necessarily connected to unwanted traces, but rather connected to a sort of ``upside-down world'' of ``stranger'' behaviours. The information carried by the negative example set diverges from that containing the positive examples but---coupled with it---can be used to understand the reasons why differences occur, ultimately providing a more accurate insight of the business process.


For this reason, we hereby focus on learning a set of constraints that is able to reconstruct which traces belong to which set, while---whenever possible---reflecting the user expectations on the quality of the extracted model according to predefined metrics. 
In particular our approach, namely \nd, aims to discover a minimal set of constraints that allow to distinguish between the two classes of examples: in this sense, it can enrich existing approaches that instead provide richer descriptions of the positive example set only. Indeed, our exploitation of the ``upside-down world'' is useful not only to better clarify what should be deemed compliant with the model and what should not, but also to better control the degree of generalisation of the resulting model, as well as to improve its simplicity.

%


The contributions of our work can be listed as follows.
\begin{itemize}
\item A novel discovery approach, \nd, based on the underlying logic semantics of Declare, which makes use of the information brought by the positive and negative example sets to produce declarative models.
\item The adoption of a satisfiability-based technique to identify the models.
\item A heuristic to select the preferred models according to input parameters dealing with generalisation or simplicity.
\item An evaluation of the performance of \nd w.r.t. other relevant works in the same field.
\end{itemize}



\section{Background}\label{sec:back}

Our technique relies on the key concept of \emph{event log}, intending it as a \emph{set} of observed process executions, logged into a file in terms of all the occurred events.
%
%
In this work, we adopt the \ac{XES} storing standard \cite{XES} for the input log. According to this standard, each \emph{event} is related to a specific process \emph{instance}, and describes the occurrence of a well-defined step in the process, namely an \emph{activity}, at a specific timestamp. The logged set of events composing a process instance is addressed as \emph{trace} or \emph{case}. 
From the analysis of the \emph{event log}, we want to extract a Declare \cite{2008-Pesic,2009-Aalst} \emph{model} of the process.
Declare is one of the most used languages for declarative process modeling. Thanks to its declarative nature, it does not represent the process as a sequence of activities from a start to an end, but through a set of constraints, which can be mapped into \ac{LTL} formulae over finite traces \cite{DBLP:journals/tweb/MontaliPACMS10,DBLP:conf/ijcai/GiacomoV13}. These constraints must all hold true when a trace complete.

Declare specifies a set of templates that can be used to model the process. 
A constraint is a concrete instantiation of a template involving one ore more process activities.
For example, the constraint \textsf{EXISTENCE(a)} is an instantiation of the template \textsf{EXISTENCE(X)}, and is used to specify that activity \textsf{a} must occur in every trace; \textsf{INIT(a)} specifies that all traces must start with \textsf{a}. \textsf{RESPONSE(a,b)} imposes that if the \textsf{a} occurs, then \textsf{b} must follow, possibly with other activities in between. 
For a description of the most common Declare templates see \cite{2008-Pesic}. 

We assume that the log contains both \emph{positive} traces---i.e., satisfying all the constraints in the business model---and \emph{negative} traces---i.e., diverging from the expected behaviour by violating at least one constraint in the (intended) model. 


\paragrafo{Language bias} Given a set of Declare templates $D$ and a set of activities $A$, we identify with $D[A]$ the set of all possible grounding of templates in $D$ w.r.t. $A$, i.e. all the constraints that can be built using the given activities.


\paragrafo{Traces and Logs} We assume that a \emph{Trace} $t$ is a \emph{finite} word over the set of activities (i.e., $t\in A^*$, where $A^*$ is the set of all the words that can be build on the alphabet defined by $A$).
Usually a log is defined as a multi-set of traces, thus allowing multiple occurrences of the same trace: the frequency of a certain trace is then considered as an indicator, for example, of the importance of that trace within the process. Since our goal is to learn a (possibly \emph{minimal}) set of constraints able to discriminate between two example classes, we rather opt to consider a \emph{Log} as a \emph{finite set} of traces. As a consequence, multiple occurrences of the same trace will not affect our discovery process.
%

\paragrafo{Declare constraints satisfaction and violation} Recalling the $LTL_f$ semantics \cite{DBLP:journals/tweb/MontaliPACMS10,DBLP:conf/ijcai/GiacomoV13}, and referring to the standard Declare semantics as in \cite{2008-Pesic}, we say that a constraint $c$ \emph{accepts} a trace $t$, or equivalently that $t$ \emph{satisfies} $c$, if $t \models c$. Similarly, a constraint $c$ rejects a trace $t$, or equivalently $t$ \emph{violates} $c$, if $t \not\models c$. Given that a Declare model $M$ is a conjunction of constraints, it follows that $M$ \emph{accepts} a trace $t$ ($t$ \emph{satisfies} $M$) if $\forall c \in M, t \models c$. Analogously, a model $M$ \emph{rejects} a trace $t$ ($t$ \emph{violates} $M$) if $\exists c \in M, t\not\models c$. In the following, we will write $t \models M$ meaning that M accepts t.

\paragrafo{Positive and negative examples} Finally, we respectively denote with $L^+$ and $L^-$ the sets of positive and negative examples (traces), reported in the input event log. We assume that:
\begin{enumerate*}[label=(\textit{\roman*})]
\item $L^+ \cap L^- = \varnothing$, and 
\item for each trace $t \in L^-$ there exists at least one grounded Declare constraint $c \in D[A]$ that accepts all the positive traces and excludes $t$.
\end{enumerate*}
In other words, we assume that the problem is feasible\footnote{Notice that sometimes real cases might not fulfill these assumptions. We will discuss this issue in section \ref{subsec:impl}}.





\section{The approach}
\label{sec:approach}

\nd aims to extract a model which correctly classifies the log traces by accepting all cases in $L^+$ and rejecting those in $L^-$\footnote{The conditions on accepting all the positives and none of the negatives can be relaxed by requiring only a percentage of them.}. Besides that, it is required to perform an abstraction step in order to be able to classify also unknown traces, which are not in the input log. 

\subsection{Definitions}
Before introducing our approach, we hereby provide some preliminary definitions that are relevant for the following explanation.



\paragrafo{Model generality.} The notion of a model accepting a trace, often referred as the compliance of a trace w.r.t. the model, naturally introduces the relation of \emph{generality} (or the converse \emph{specificity}) between models.
Intuitively, a model $M$ is more general than another $M'$ if $M$ accepts a superset of the traces accepted by $M'$. That is, denoting with $T_M$ the set of all traces compliant with a model $M$, $M$ is more general than another model $M'$---and symmetrically $M'$ is more specific than $M$---if and only if $T_{M'} \subseteq T_M$. 
More precisely, we say that 
\theoremstyle{definition}\label{def:generality}
\begin{definition}{}
a model $M\subseteq D[A]$ is \emph{more general} than $M'\subseteq D[A]$ (written as $M \succeq M'$) when for any $t\in A^*$, $t\models M' \Rightarrow t\models M$ , and \emph{strictly more general} (written as $M \succ M'$) if $M$ is more general than $M'$ and there exists a $t'\in A^*$ s.t.\ $t'\not\models M'$ and $t'\models M$.
\end{definition}

Note that this definition is consistent with that of subsumption between Declare templates provided in Di Ciccio et al. \cite{2017-DiCiccio}. Indeed, Declare templates can be organised into a subsumption hierarchy according to the logical implications that can be derived from their semantics.
\begin{example}{}
The constraint \textsf{INIT(a)} accepts only traces that start with \textsf{a}. Hence, \textsf{a} exists in each one of those accepted traces. In other words, all those traces satisfy also the constraint \textsf{EXISTENCE(a)}. However, the latter constraint accepts also traces that contains \textsf{a} even if they do not start with \textsf{a}. This relation is valid irrespectively of the involved activity. In a sense, we could say that the template \textsf{EXISTENCE(X)} is \emph{more general} than \textsf{INIT(X)}.
\end{example}
This idea is frequently expressed through the subsumption operator $\sqsupseteq$. Given two templates $d, d' \in D$, we say that $d$ \emph{subsumes} $d'$, i.e. $d$ \emph{is more general than} $d'$ (written $d\sqsupseteq d'$), if for any grounding of the involved parameters w.r.t. the activities in $A$, whenever a trace $t \in A^*$ is compliant with $d'$, it is also compliant with $d$ \cite{2017-DiCiccio} .

\begin{remark}{}\label{re:subset-generality}
For any pair of models $M, M'\subseteq D[A]$, $M\subseteq M'$ implies that $M$ is more general than $M'$ ($M\succeq M'$). This stems from the Declare semantics \cite{2008-Pesic} on $LTL_f$ \cite{DBLP:conf/ijcai/GiacomoV13}.
\end{remark} 

Unfortunately, the opposite implication does not hold, i.e. if we have $M, M'\subseteq D[A]$ such that $M\succeq M'$, we cannot guarantee that $M\subseteq M'$. A clear example is $M=\{\textsf{EXISTENCE(a)}\}$ and $M'=\{\textsf{INIT(a)}\}$.

\paragrafo{Initial model} A wide body of research has been devoted to techniques to mine declarative process models that characterise a given event log (our positive traces). Our approach can leverage these techniques and refine their results by taking into account the negative examples as well. To this end, we consider a---possibly empty---\emph{initial model} $P$, i.e. a set of Declare constraints that are known to characterise the positive traces. For example, such set can be the expression of domain knowledge or the result of a state-of-the-art discovery algorithm previously applied to $L^+$. To apply our technique we only require that all the positive traces are compliant with all the constraints in $P$. 
We are aware that often state-of-the-art approaches do not emit a model compliant with all the traces in the input log. In these cases, we consider as positive only the traces that are allowed by $P$.

\paragrafo{Candidate solution} As the goal of our technique is to refine the initial model $P$ taking into account both positive and negative traces, we can define which are the necessary conditions for a set of constraints $S$ to be a candidate solution for our discovery task.

\theoremstyle{definition}
\begin{definition}{}\label{def:cand}
Given the initial model $P$, a \emph{candidate solution} for the discovery task is any $S\subseteq D[A]$ s.t.
\begin{enumerate} [label=\textit{(\roman*)}]
  \item $P\subseteq S$;
  \item $\forall t\in L^+$ we have $t\models S$;
  \item $\forall t\in L^-$ we have $t\not\models S$.
\end{enumerate}
\end{definition}


\paragrafo{Optimality criterion} Clearly, there can be several sets satisfying these conditions. They differ from the way they classify the unknown traces, which are not in $L^+$, nor in $L^-$. Therefore, we need to introduce some way to compare the multiple output models in order to identify the preferable ones.
%
In some context, \emph{generality} can be a measure of the quality of the solution, i.e. we want to identify the set that is less committing in terms of restricting the admitted traces. In some other context on the contrary, we might be interested in the identification of a more specific model. So besides allowing all traces in $L^+$ and forbidding all traces in $L^-$, the choice between a general or specific model, obviously affects the classification of the unknown traces. Alternatively, \emph{simplicity} is another criterion: one can be interested in the most \emph{simple} solution, i.e. the one that is presumed to be easier to understand irrespectively from the degree of generality/specificity it accomplishes.


Let us focus on \emph{generality} for the moment. In this case, we are interested in the candidate solution $S$ (i.e., satisfying the properties of Definition \ref{def:cand}) such that there is no other candidate solution $S'\subseteq D[A]$  strictly more general than $S$ (i.e., $\nexists~ S'$ s.t. $S\prec S'$).

Although testing for strict generality between two set of constraints is a decidable problem, its worst case complexity makes an exact algorithm unfeasible because, recalling definition \ref{def:generality}, it would require to asses the compliance of any trace $t \in A^*$ with the two models going to be compared.
For this reason, we propose an alternative method based on comparing the logical consequences that can be deducted from the models.

The method makes use of a set of deduction rules which account for the \emph{subsumption} between Declare templates. Our work integrates the rules introduced in \cite{2017-DiCiccio}, into a function, namely the \emph{deductive closure operator}, which satisfies the properties of extensivity, monotonicity, and idempotence.

\theoremstyle{definition}\label{def:closure}
\begin{definition}{}
Given a set $R$ of subsumption rules, a \emph{deductive closure operator} is a function $cl_R: \mathcal{P}(D[A])\rightarrow\mathcal{P}(D[A])$ that associates any set $M \in D[A]$ with all the constraints that can be logically derived from $M$ by applying one or more deduction rules in $R$.
\end{definition}

\begin{example}
Let be:
\begin{itemize}
\item $R=\{\ \textsf{EXISTENCE(X)} \sqsupseteq \textsf{INIT(X)}\ \}$
\item $M'=\{ \textsf{INIT(a)}\ \}$
\end{itemize}
If we apply the deductive closure operator $cl_R$ to $M'$, we get:
$$
cl_R(M') = \{\ \textsf{INIT(a), EXISTENCE(a)}\ \}
$$
\end{example}

\noindent Obviously, we are interested in sets $R$ of \emph{correct} subsumption rules, that is for any $M\subseteq D[A]$ and $t\in A^*$, $t\models M$ iff $t\models cl_R(M)$. In the rest of the paper, for the easy of understanding, we will omit the set $R$ and we will simply write $cl(M)$. The complete set of employed rules is available\footnote{The repository is freely available at \url{https://zenodo.org/record/5158528}, and the experiments can be run through a Jupyter Notebook at \url{https://mybinder.org/v2/zenodo/10.5281/zenodo.5158528/}}.

As the closure of a model is again a subset of $D[A]$, Remark \ref{re:subset-generality} is also applicable, i.e. for any $M, M'\subseteq D[A]$, $cl(M)\subseteq cl(M')$ implies that $M$ is more general than $M'$ ($M\succeq M'$).
%
%
Thanks to this property, the deductive closure operator can be used to compare Declare models w.r.t.~generality. To provide an intuition, let us consider the following example:

\begin{example}
Let be:
\begin{itemize}
\item $R=\{\ \textsf{EXISTENCE(X)} \sqsupseteq \textsf{INIT(X)}\ \}$
\item $M = \{\ \mathsf{EXISTENCE(a)}\ \}$
\item $M' = \{\ \mathsf{INIT(a)}\ \}$
\end{itemize}
We cannot express any subset relation between $M$ and $M'$, thus making Remark \ref{re:subset-generality} inapplicable.
Nonetheless, if we take into account their closure, we have:
\begin{align*}
	& cl(M)=\{\mathsf{EXISTENCE(a)}\}\quad \textnormal{and} \\
	& cl(M')=\{\textsf{INIT(a)}, \textsf{EXISTENCE(a)}\}
\end{align*}
As $cl(M)$ is a subset of $cl(M')$, we conclude that $M$ is more general than $M'$.
\end{example}
\noindent In other words, the closure operator (being based on the subsumption rules) captures the logical consequences deriving from the Declare semantics.
Due to the nature of the Declare language we cannot provide a complete calculus for the language of conjunctions of Declare constraints. For this reason, we cannot guarantee the strictness, nor the opposite implication (i.e. $M'\preceq M$ does not implies $cl(M)\subseteq cl(M')$). Anyway, the closure operator provide us a powerful tool for confronting candidate solutions w.r.t. the generality criterion.

\subsection{Two-step procedure}

In this section we introduce the theoretical basis of our \nd approach. The parameters of the problem are the following.
\begin{itemize}
  \item $D$ set of Declare templates (language bias)
  \item $A$ set of activities
  \item $cl$ closure operator equipped with a set $R$ of subsumption rules. $cl: \mathcal{P}(D[A])\rightarrow\mathcal{P}(D[A])$
  \item $L^+$ log traces labelled as positive. $L^+ \subseteq A^*$
  \item $L^-$ log traces labelled as negative. $L^- \subseteq A^*$
  \item $P$ initial model. $P\subseteq \{c\in D[A]\quad | \quad \forall t\in L^+,\, t\models c\}$
\end{itemize}

For the sake of modularity and easiness of experimenting with different hypotheses and parameters, we divide our approach into two clearly separate stages: the first which identifies the candidate constraints, and a second optimisation stage which selects the solutions. However, these two steps are merged into a single monolithic search-based approach. 

Starting from the set of all constraints $D[A]$, the first stage aims at identifying all the constraints of $D[A]$ which accept all positive traces and reject at least a negative one. To this end, it first computes the set of constraints that accepts all traces in $L^+$, namely the set \emph{compatibles}. 
\begin{equation}
{compatibles(D[A], L^+)} = \{c\in D[A]~|~\forall t\in L^+,~ t\models c \} \\
\end{equation}
For simplicity of the notation, since the set of constraints $D[A]$ and the log $L^+$ are given, we will omit them in the following.
The \emph{compatibles} set is then used to build a \textit{\sheriff} function that associates to any trace $t$ in $L^-$ the constraints of \textit{compatibles} that rejects $t$.
The result is therefore a function with domain $L^-$ and co-domain $\mathcal{P}({compatibles})$ s.t.:
\begin{equation}
{\textit{\sheriff}}(t) = \{c\in {compatibles}~|~t\not\models c\} \\
\end{equation}

The second stage aims at finding the optimal solution according to some criterion. Therefore, it starts by computing two sets $\mathcal{C}$ and $\mathcal{Z}$. 
Let $\mathcal{C}$ be the set of those constraints in $D[A]$ that accept all positive traces and reject at least one negative trace. Such set can be derived from the \textit{\sheriff} function as: $\mathcal{C} = \bigcup_{t\in L^-} \textit{\sheriff}(t)$.
Let $\mathcal{Z}$ be all the subsets of $\mathcal{C}$ excluding all negative traces\footnote{As we will discuss in section \ref{subsec:impl}, the implementation must take into account that it might not be always possible to find models fulfilling Eq.\ref{eq:mathcalS}}, i.e.,
\begin{equation}\label{eq:mathcalS}
\mathcal{Z}=\{M\in\mathcal{P}(\mathcal{C})\mid \forall t\in L^-~t\not\models M \} 
\end{equation}

\begin{example}
Let $L^+=\{\mathsf{ab}\}$, $L^-=\{\mathsf{a, b, ba}\}$, and $D=\{\mathsf{EXISTENCE(X)},\mathsf{RESPONSE(Y,Z)}\}$.
The set of activities is $A=\{\ \mathsf{a}, \mathsf{b}\ \}$.
The grounded set of constraints is then $D[A] = \{$ $\mathsf{EXISTENCE(a)}$, $\mathsf{EXISTENCE(b)}$, $\mathsf{RESPONSE(a,b)}$, $\mathsf{RESPONSE(b,a)}\ \}$.

The \textit{compatibles} set would be:
\begin{align*}
compatibles(D[A], L^+) =  \{\ & \mathsf{EXISTENCE(a)},\ \mathsf{EXISTENCE(b)},\\
&  \mathsf{RESPONSE(a,b)} \ \}	
\end{align*}
and the computation of the \textit{\sheriff} function finds:
\begin{subequations}
 \begin{align*}
     &\textit{\sheriff}(\mathsf{a}) =\{\mathsf{EXISTENCE(b)},\ \mathsf{RESPONSE(a,b)}\} \\
     &\textit{\sheriff}(\mathsf{b}) =\{\mathsf{EXISTENCE(a)}\} \\
     &\textit{sheriff}(\mathsf{ba})=\{\mathsf{RESPONSE(a,b)}\}
 \end{align*}
\end{subequations}
In this case, there are two subsets of $\mathcal{C}$ excluding all traces in $L^-$, i.e., $\mathcal{Z}=\{M_1,M_2\}$, where
\begin{subequations}
 \begin{align*}
     & M_1=\{\ \mathsf{EXISTENCE(b)},\ \mathsf{EXISTENCE(a)},\ \mathsf{RESPONSE(a,b)}\ \} \\
     & M_2=\{\ \mathsf{RESPONSE(a,b)},\ \mathsf{EXISTENCE(a)}\ \}
 \end{align*}
\end{subequations}
\end{example}

Once $\mathcal{C}$ and $\mathcal{Z}$ are computed, the goal of the optimisation step is to select the ``best'' model in $\mathcal{Z}$ which can be either devoted to \emph{generality/specificity}, or \emph{simplicity}.
When the most general model is desired, the procedure selects as solution the model $S\in \mathcal{Z}$ such that 
\begin{subequations}
  \begin{align}
    \text{there is no $S'\in\mathcal{Z}$ s.t. } cl(S'\cup P)\subset cl(S\cup P) \label{eq:most-gen}\\
    \text{there is no $S'\subset S$ s.t. } cl(S'\cup P)=cl(S\cup P)\label{eq:redun}
  \end{align}
\end{subequations}
The first condition, Eq. \eqref{eq:most-gen}, ensures generality by selecting the model $S$ for which the logical consequences of $S\cup P$ are the less restricting. 
In this way, the initial model $P$ (containing a set of Declare constraints that are known to characterise $L^+$) is enriched taking into account the information derived by $L^-$.
Furthermore, since from the point of view of generality we are not interested in the content of the selected model, but rather in its logical consequences, the closure operator $cl$ ensures that no other model in $\mathcal{Z}$ is more general than the chosen $S$. 
The second condition, Eq. \eqref{eq:redun}, allows to exclude redundancy inside the selected model $S$ by ensuring that it does not contain constraints that are logical consequence of others in $S$. Considering the previous example, this optimisation step allows to chose model $M_2$ as solution because $\mathsf{EXISTENCE(b)}$ is a logical consequence of $\mathsf{EXISTENCE(a)} \land  \mathsf{RESPONSE(a,b)}$.

If we were interested in the less general model, condition \eqref{eq:most-gen} would be 
\begin{equation}\label{eq:most-spe}
\text{there is no $S'\in\mathcal{Z}$ s.t. } cl(S'\cup P)\supset cl(S\cup P)
\end{equation}
whereas the redundancy constraint would be ensured through the same Eq. \eqref{eq:redun} because, even when we look for the most specific model, redundancy compromises its readability, without adding any value.

Generality/specificity is not the only desirable optimality criterion. If we are interested in the \emph{simplest} model instead, a solution composed of a limited number of constraints is certainly preferable. So, we also experimented with an alternative optimisation formulation based on the set cardinality. The procedure selects the $S \in\mathcal{Z}$ such that:
\begin{subequations}
   \begin{align}
    \text{there is no } S'\in\mathcal{Z} \text{ s.t. } & |cl(S'\cup P)| < |cl(S\cup P)| \label{eq:simpl1}\\
	\begin{split}
    \text{there is no } S'\in\mathcal{Z} \text{ s.t. } & |cl(S'\cup P)|=|cl(S\cup P)| \text{ and}\\
	&  |S'| < |S| \label{eq:simpl2} 
	\end{split}
   \end{align}
\end{subequations}
where the first equation selects the set with the smaller closure, whereas the second allows to choose the solution with less constraints among those with closure of equal cardinality.


\theoremstyle{definition}\label{th:subset-generality}
\begin{theorem}{}
The models that are solution according to the simplicity criterion are also solutions for the generality criterion.
\end{theorem} 
\begin{proof}
Suppose ad absurdum that there is a model $S\in\mathcal{Z}$ that is optimal according to the simplicity criterion of Eq.~\eqref{eq:simpl1} and \eqref{eq:simpl2} but it is not the most general, i.e. either Eq.~\eqref{eq:most-gen} or Eq.~\eqref{eq:redun} are violated for $S$. If $S$ violated Eq.~\eqref{eq:most-gen}, it would exists an $S'\in\mathcal{Z}$ s.t. $cl(S'\cup P)\subset cl(S\cup P)$. But clearly, this implies that $|cl(S'\cup P)| < |cl(S\cup P)|$, which contradicts Eq.~\eqref{eq:simpl1}. On the other hand, if $S$ violated Eq.~\eqref{eq:redun}, it would exists an $S' \subset S$ s.t. $cl(S'\cup P)=cl(S\cup P)$. Obviously we would also have $|S'| < |S|$ and $|cl(S'\cup P)|=|cl(S\cup P)|$, which contradict Eq.~\eqref{eq:simpl2}. 
\end{proof}

Conversely, the opposite implication in Theorem \ref{th:subset-generality} does not necessarily hold. Indeed, let assume that $P$ is empty and $cl$ is the identity function\footnote{This may also be the case when the constraints selected in the first stage are logically independent.}; consider two negative traces $t_1, t_2$ and a \textit{\sheriff} function producing three constraints $\{c_1, c_2,c_3\}$. In particular, $\textit{\sheriff}(t_1)=\{c_1, c_2\}$ and $\textit{\sheriff}(t_2)=\{c_1, c_3\}$. The only simplicity-oriented solution would be  $\{c_1\}$, whereas as regards the generality-oriented solutions we would have both $\{c_1\}, \{c_2, c_3\}$. We must remark that the simplicity principle is based on the intuition that ``smaller'' Declare models should be easier to understand for humans. However, we might notice that, since the two models $\{c_1\}$ and $\{c_2, c_3\}$ are not directly comparable according to their semantics, deciding which is the ``best'' might depend on the constraints themselves, ad well as the specific domain.


\subsection{Implementation}
\label{subsec:impl}

The first stage is implemented via the Algorithm \ref{algcand}, which starts by collecting the set $compatibles$ of the constraints that accept (are satisfied by) all the positive traces (Line \ref{algcand:candidates}). Subsequently, each negative trace is associated (by means of the function \textit{\sheriff}) with those constraints in $compatibles$ that reject (are violated by) a trace in $L^-$ (Line \ref{algcand:choices}). Notice that for some negative example, i.e. for some trace $t\in L^-$ we might have $\textit{sheriff}(t)=\varnothing$ because the chosen language may not be expressive enough to find a constraint able to reject $t$ while admitting all traces of $L^+$. This situation might arise also in case $t$ belongs to both $L^+$ and $L^-$.

\makeatletter
\algrenewcommand\ALG@beginalgorithmic{\footnotesize}
\makeatother

\begin{algorithm}
    \caption{Identification of the constraints accepting all traces in $L^+$ and rejecting at least one trace in $L^-$.}
    \label{algcand}
    \textbf{Input:}  $D[A], L^+, L^-$\\
    \textbf{Output:} $\textit{\sheriff} : L^- \rightarrow \mathcal{P}({D[A]})$
    	\begin{algorithmic}[1] 
   \Procedure{\sheriff Generation}{$D[A],\, L^+,\, L^-$} 
   	\State ${compatibles}= \{c \in D[A]| \forall t \in L^+,\, \Call{compliant}{t,c} = \texttt{True}\}$ 
	\label{algcand:candidates}
	\For {$t \in L^-$}
		\State $\textit{\sheriff}(t) = \{c \in {compatibles} | \Call{compliant}{t,c} = \texttt{False}\}$\label{algcand:choices}
	\EndFor
	\State \Return \textit{\sheriff}
    \EndProcedure
    \end{algorithmic}
\end{algorithm}

\makeatletter
\algrenewcommand\ALG@beginalgorithmic{\normalsize}
\makeatother

The implementation of the compliance verification \textproc{compliant} (i.e.\ $t\models c$) leverages the semantics of Declare patterns defined by means of regular expressions \cite{2017-DiCiccio}
~to verify the compliance of the traces. It is implemented in Go language employing a regexp implementation that is guaranteed to run in time linear in the size of the input\footnote{For more details, see the Go package regexp documentation at \url{https://golang.org/pkg/regexp/}}.
\lstset{language=Prolog}

The second optimisation stage has been implemented using the \ac{ASP} system \textsc{Clingo}~\cite{DBLP:journals/corr/GebserKKS14}. The main reason for selecting an \ac{ASP} system for finite domain optimisation is that rules provide an effective and intuitive framework to implement a large class of closure operators. Indeed, all the deductive systems for Declare that we analysed in the literature (see e.g.~\cite{2016-Bernardi,2017-DiCiccio}) can be recasted as Normal Logic Programs~\cite{2008-Lifschitz} by exploiting the assumption that the set of activities is finite and known in advance.
For example the valid formula $\textsc{Init}(a)\implies\textsc{Precedence}(a,b)$ that holds for any pair of activities $a, b$ can be written as the rule
\begin{lstlisting}
  precedence(A,B) :- init(A), activity(B).
\end{lstlisting}
using a specific predicate (\lstinline{activity}) holding the set of activities.


The second stage is implemented as described in Algorithm \ref{alg:select}. The required input parameters, properly encoded as an \ac{ASP} program, are the initial model $P$, the \textit{\sheriff} function computed by Algorithm \label{alg:cand} and a custom function ${is\_better}:\mathcal{P}({D[A]})\times\mathcal{P}({D[A]})\rightarrow \{\texttt{True},\ \texttt{False}\}$. The purpose of the latter is to implement the chosen optimality criterion by taking as input two constraint sets and providing as output a boolean value representing whether the first set is better than the second. If the two sets are not comparable according to the criterion, ${is\_better}$ returns \texttt{False}. Indeed, such a function is the expression of the global or partial ordering that the optimality criterion induces on the solutions.

\begin{algorithm}
    \caption{Selection of the best solutions according to a custom criterion.}
    \label{alg:select}
    \textbf{Input:}  $P,\, \textit{\sheriff} : L^- \rightarrow \mathcal{P}({D[A]})$,$\,{is\_better}:\mathcal{P}({D[A]})\times\mathcal{P}({D[A]})\rightarrow \{\texttt{T},\texttt{F}\}$\\
    \textbf{Output:} $\mathcal{Z}$, i.e. the set of the best solutions
	\begin{algorithmic}[1] 
   \Procedure{Selection}{$\textit{\sheriff},\, P$} 
   	\State $\mathcal{Z}=\varnothing$
  	\State $ L'^- = \{t \in L^- \: | \: \textit{\sheriff}(t) \neq \varnothing \land t\models P\}$
	\State $\mathcal{C} = \bigcup_{t\in L'^-} \textit{\sheriff}(t)$\label{alg:buildC}
	\State \textbf{for each} $S \subseteq \mathcal{C}$ \textbf{s.t.} $\forall t \in L'^-, \; S\cap \textit{\sheriff}(t) \neq \varnothing$ \textbf{do} \label{alg:subsetS}
	\Indent
		\If{$(\forall S'\in\mathcal{Z} \ !is\_better(S',S)\,)$} \label{alg:isbetter}
		\State $\mathcal{Z} \, \leftarrow\, \{S\} \cup \{ S'\in\mathcal{Z}\ |\ !is\_better(S,S')\,\}$ \label{alg:previousAreOk}
		\EndIf
	\EndIndent
	\State \Return $\mathcal{Z}$  
    \EndProcedure
    \end{algorithmic}
\end{algorithm}

The algorithm starts by computing a set $L'^-$ of all those negative traces that can be excluded by at least a constraint ($\textit{\sheriff}(t)\neq\varnothing$) and are still instead accepted by the initial model $P$ ($t\models P$). 
Indeed, albeit from the theoretical point of view we assumed that for each trace $t \in L^-$ there exists at least one Declare constraint in $D[A]$ accepting all positive traces and discarding $t$, real cases might not fulfil this assumption.
When it is not possible to exclude a negative trace $t$, Algorithm \ref{algcand} returns $\textit{\sheriff}(t)=\varnothing$, and Algorithm \ref{alg:select} overlooks this case by computing $L'^-$. 
The set $\mathcal{C}$ is then build (Line \ref{alg:buildC}) by considering the constraints allowing all traces in $L^+$ and disallowing at least one trace in $L'^-$ (Line \ref{alg:buildC}). From that, the algorithm selects any subset $S$ fulfilling the condition $\forall t \in L'^- \; S\cap {choices}(t) \neq \varnothing$ (Line \ref{alg:subsetS}), i.e., any $S$ accepting all positive traces and rejecting all the negatives that can be actually excluded.

Any such $S$ is then included in the solution set $\mathcal{Z}$ if the latter does not contain another solution $S'$ that is better than $S$ according to the custom optimality criterion expressed by the ${is\_better}$ operator (Line \ref{alg:isbetter}). Solutions $S'$ previously founded are kept into $\mathcal{Z}$ only if the newly found solution $S$ is not better than them (Line \ref{alg:previousAreOk}). Notice that both Line \ref{alg:isbetter} and \ref{alg:previousAreOk} make use of the function $is\_better$ in a negated form: this is due to the fact that, according to the chosen optimality criterion, sometimes it would not be possible to compare two solutions.

Regarding the optimality criterion, for example we could consider \emph{generality}. In that case, ${is\_better}$ employs the criteria of Eq. \eqref{eq:most-gen} (for generality) and \eqref{eq:redun} (to avoid redundancy). Conversely, if we are interested in the most specific solution, ${is\_better}$ must implement Eq. \eqref{eq:most-spe} and \eqref{eq:redun}.
Finally, when the optimality criterion is \emph{simplicity}, Eq. \eqref{eq:simpl1} and \eqref{eq:simpl2} must be used.

In real cases, the returned set $\mathcal{Z}$ might contain several solutions. 
If the number of solutions provided by the procedure is too high for human intelligibility, the optimality condition could be further refined by inducing a preference order in the returned solution. For example, among the most general solutions one can be interested in being reported first those models with the lower number of constraints, or with certain Declare templates. The advantage of our approach is precisely in the possibility to implement off-the-shelves optimisation strategies, where---adapting the $is\_better$ function or even the definition of the closure operator---the developer can easily experiment with different criteria.

\begin{example}
Consider the sets of positive and negative examples composed by only one trace each: $L^+=\{\textsf{bac}\}$ and $L^-=\{\textsf{ab}\}$. Suppose also that:
\begin{itemize}
\item $P=\varnothing$;
\item $D=\{\ \textsf{EXISTENCE(X), INIT(X)}\ \}$;
\item the alphabet of activities is just $A=\{\textsf{a, b, c}\}$.
\end{itemize}
Then, the set of ground constraints can be easily elicited: $D[A]=\{$ \textsf{EXISTENCE(a), EXISTENCE(b), EXISTENCE(c), INIT(a), INIT(b), INIT(c)}$\ \}$.

If we want to learn the most general model, Algorithm \ref{algcand} elects the following compatible constraints:
\begin{align*}
	{compatibles}=\{ & \textsf{EXISTENCE(a), EXISTENCE(b)},\\
	& \textsf{EXISTENCE(c), INIT(b)}\}
\end{align*}
\noindent and emits:
$$\textit{\sheriff}(\textsf{ab})=\{\ \textsf{EXISTENCE(c), INIT(b)}\ \}.$$

In this simple case, the subsets satisfying the condition of Line \ref{alg:subsetS} in Algorithm \ref{alg:select} would be: 
\begin{align*}
 S_1= &  \{\textsf{EXISTENCE(c)}\}  \\
 S_2= &  \{\textsf{INIT(b)}\}   \\
 S_3= &  \{\textsf{EXISTENCE(c)}, \textsf{INIT(b)}\} \\
 S_4= & \{\textsf{EXISTENCE(c)}, \textsf{EXISTENCE(a)}\} \\
  ... & \\
 S_n= & \{\textsf{EXISTENCE(c)}, \textsf{INIT(b)}, \textsf{EXISTENCE(a)}\} \\
 ... &
\end{align*}

As we are interested in the most general models, both $S_1=$ $\{\ \textsf{EXISTENCE(c)}\ \}$ and $S_2=\{\ \textsf{INIT(b)}\ \}$ are optimal solutions. Note that these two solutions cannot be compared according to the definitions of generality because there exist traces (such as the unknown trace $\textsf{b}$) compliant with $S_2$ and non-compliant with $S_1$ (i.e., there is no subset relation between ${C}_{S_1}$ and ${C}_{S_2}$).
Obviously, the choice of one model over another influences the classification of all those unknown traces that---being not part of the input log---are not labelled as positive of negative.

If we were interested in the most simple solution instead, $S_1=\{\ \textsf{EXISTENCE(c)}\ \}$ would have been our choice, because its closure is the smaller in cardinality.

Finally, if we were interested in the most specific set of constraints, the application of a convenient ${is\_better}$ function would have determined the choice of $\{$\textsf{EXISTENCE(a), EXISTENCE(c), INIT(b)}$\}$---where the redundancy check of Eq. \eqref{eq:redun} operated by discarding \textsf{EXISTENCE(b)}. 
\end{example}

\newcommand{\todoinfc}[1]{\todo[inline,backgroundcolor=yellow]{FC: #1}}

\section{Experimental evaluation}
\label{sec:eval}


One of the difficulties of evaluating process mining algorithms is that given a log, the underlying model might not be known before. As a consequence, it might be difficult to establish an \emph{ideal model} (a golden standard) to refer and confront with. In this regard, a number of metrics and evaluation indexes have been proposed in the past to evaluate how a discovered model fits a given log \cite{2015-Adriansyah,2014-Broucke,2018-Ponce}. However, those metrics might provide only a partial answer to the question of ``how good'' is the discovered model. 
In the case of \nd, a further issue influences the evaluation process: the difficulty of performing a ``fair'' comparison with existing techniques because the majority of the methods we could access have been designed to use ``positive'' traces only.


We pursued two different evaluation strategies. On one side, we defined a model, and from that model we generated a synthetic, artificial log, taking care that it exhibits a number of desired properties: in a sense, this part of the evaluation can be referred as being about a ``controlled setting''. A first aim is to understand if \nd succeeds to discover a \emph{minimum} set of constraints for distinguishing positive from negative examples; a second aim is to qualitatively evaluate the discovered model, having the possibility to confront it with the original one. Experiments conducted on that synthetic log are reported and discussed in Section \ref{sec:syntheticlog}.

On the other side, we applied \nd to some existing logs, thus evaluating it on some real data set. Again, this experiment has two aims: to understand weakness and strengths of \nd w.r.t. to some relevant literature; and to confront the proposed approach with real-world data---and difficulties that real-world data bring along. Section \ref{sec:realdata} is devoted to present the selected logs and discuss the obtained results.

The source code and the experiments are available\footnote{The repository is published at \url{https://zenodo.org/record/5158528}, and the experiments can be run through a Jupyter Notebook at \url{https://mybinder.org/v2/zenodo/10.5281/zenodo.5158528/}}.


\subsection{Experiments on a synthetic dataset}
\label{sec:syntheticlog}

%


The synthetic log has been generated starting from a Declare model, using a tool \cite{2020-Loreti} based on Abductive Logic Programming. The model has been inspired by the Loan Application process reported in \cite{DBLP:books/sp/DumasRMR18}. In our model, the process starts when the \emph{loan application} is received. Before \emph{assessing the eligibility}, the bank proceeds to \emph{appraise the property} of the customer, and to \emph{assess the loan risk}. Then, the bank can either \emph{reject the application} or \emph{send the acceptance pack} and, optionally, \emph{notify the approval} (if not rejected). During the process execution the bank can also \emph{receive positive} or \emph{negative feedback} (but not both), according to the experience of the loan requester. It is not expected, however, that the bank receives a \emph{negative feedback} if the \emph{acceptance pack} has been sent. Moreover, due to temporal optimization, the bank requires that the \emph{appraise of the property} is done before \emph{assessing the loan risk}.
To ease the understanding of the loan application process, a Declare model of the process is reported in Fig. \ref{fig:ex}. Moreover, all the activities have been constrained to either not be executed at all, or to be executed at most once: in Declare terminology, all the activities have been constrained to \textsf{absence2(X)}.

\begin{figure}[t]
\centering
\includegraphics[width=0.6\columnwidth]{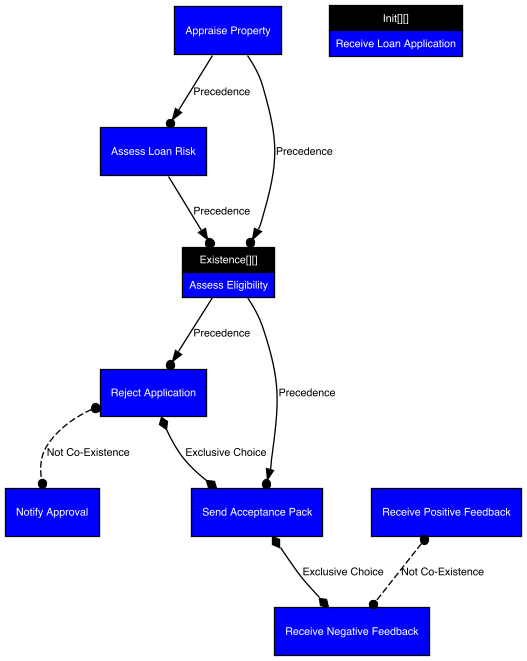}
\caption{Loan approval declare process model.}
\label{fig:ex}
\end{figure}

To test \nd, besides positive traces, we generated also negative traces. In particular, we generated traces that violate two different constraints:
\begin{enumerate}[label=(\textit{\alph*})]
\item the \textsf{precedence(assess\_loan\_risk, assess\_eligibility)}, that is violated when either the latter activity is executed and the former is absent, or if both the activities appear in the log, but in the wrong temporal order;
\item the \textsf{exclusive\_choice(send\_acceptance\_pack, receive\_negative\_feedback)}, that is violated when a trace either contains both the activities, or does not contain any of them.
\end{enumerate}
The resulting log consists of 64,000 positives traces, 25,600 traces that violate the constraint as in $(a)$, and 10,240 traces that violate the constraint as specified in $(b)$.
When fed with the positives traces and traces violating the constraint in $(a)$, \nd successfully manages to identify constraints that allow to clearly distinguish positives from negatives traces. Moreover, the discovered constraint coincides with the one we originally decided to violate during the generation phase. When confronted with the scenario $(b)$, \nd again successfully managed to identify a minimum model able to discriminate between positive and negative traces, and the identified constraint is indeed logically consistent with the constraint originally selected for the violation.
Table \ref{tab:syntResults} summarize the obtained results and reports the first selected model for each scenario.

\begin{table*}
\scalebox{.87}{
\begin{tabular}{c c c c l l}
\toprule
Scenario & Positive Trace \# & Negative Trace \#  & Time & Originally Violated Constraint & First Discovered Model \\
\midrule
$\multirow{4}{*}{(a)}$ & \multirow{4}{*}{64,000} & \multirow{4}{*}{25,600} & \bf{Total: \fpeval{81.78 + 1.94 + 109.74 + 3.32 + 15.165}s}  \\
& & & Compatibles:  \fpeval{81.78 + 1.94}s & \emph{precedence(assess\_loan\_risk,}  & \emph{precedence(assess\_loan\_risk,} \\
& & & Choices:  \fpeval{109.74 + 3.32}s  & \emph{assess\_eligibility)} & \emph{assess\_eligibility)} \\
& & & Optimisation: 15.165s & & \\
\midrule
$\multirow{4}{*}{(b)}$ & \multirow{4}{*}{64,000} & \multirow{4}{*}{10,240} & \bf{Total: \fpeval{81.78 + 1.94 + 94.51 + 2.96 + 1.379}s} \\
& & & Compatibles:  \fpeval{81.78 + 1.94}s & \emph{exclusive\_choice(send\_acceptance\_pack,} & \emph{coExistence(reject\_application,} \\
& & & Choices:  \fpeval{94.51 + 2.96}s  & \emph{receive\_negative\_ feedback)} & \emph{receive\_negative\_feedback)}\\
& & & Optimisation: 1.379s & & \\
\bottomrule
\end{tabular}
}
\caption{Models discovered when dealing with the synthetic data set.}
\label{tab:syntResults}
\end{table*}


For the sake of completeness, we decided to experiment also with the Process Discovery Tool of the Rum Framework\footnote{\url{https://rulemining.org/}}, that is based on the Declare Miner algorithm \cite{2018a-Maggi}. Based on the exploitation of positive traces only, Declare Miner discovers a rich model that describes as ``most exactly'' as possible the given traces. When fed with the positive traces of our artificial log, and with the \emph{coverage} parameter set to 100\% (i.e., prefer constraints that are valid for all the traces in the logs), the RuM Framework discovers a model made of 514 constraints. If the coverage is relaxed to 80\% (prefer constraints that are satisfied by at least the 80\% of the traces), the model cardinality grows up to 1031 constraints.

In both cases the discovered model is able to distinguish between the positive and the negative traces. This is not surprising, since Declare Miner aims to identify all the constraints that hold for a given log: hence, it will discover also those constraints that allow to discern positive from negative traces. Rather, this result is a clear indication that indeed our artificial log has been constructed ``correctly'', since negative traces differ from the positive ones for some specific constraints, and the positive traces exhaustively elicit the cases that can occur. 
This is typical of artificial logs, while real-life logs might not enjoy these properties.

Another consideration is about the cardinality of the discovered model: the Declare Miner approach provides a far richer description of the positive traces, at the cost perhaps of bigger models. Our approach instead has the goal of identifying the \emph{smallest} set of constraints that allow to discriminate between positive and negatives. In this sense, approaches like the one presented in this paper and Declare Miner are complementary.


\subsection{Evaluation on case studies from real data}
\label{sec:realdata}


For the experimentation with real datasets, we used three real-life event logs: \textsc{cerv}, \textsc{sepsis} and \textsc{bpic12}. Starting from these event logs we generated 5 different datasets, each composed of a set of positive and a set of negative traces, by applying different criteria to distinguish between positive and negative traces, i.e., by labeling the event log with different labeling functions. 

\textsc{cerv} is an event log related to the process of cervical cancer screening carried out in an Italian cervical cancer screening center~\cite{2007b-Lamma}. Cervical cancer is a disease in which malignant (cancer) cells form in the tissues of the cervix of the uterus. The screening program proposes several tests in order to early detect and treat cervical cancer. It is usually composed by five phases: Screening planning; Invitation management; First level test with pap-test; Second level test with colposcopy, and eventually biopsy. The traces contained in the event log have been analyzed by a domain expert and labeled as compliant (positive traces) or non-compliant (negative traces) with respect to the cervical cancer screening protocol adopted by the screening center.

\textsc{sepsis}~\cite{Sepsis} is an event log that records trajectories of patients with symptoms of the life-threatening sepsis condition in a Dutch hospital.
Each case logs events since the patient's registration in the emergency room until her discharge from the hospital. Among others, laboratory tests together with their results are recorded as events. The traces contained in the event log have been labelled based on their cycle execution time. In the \textsc{sepsis$_{mean}$} dataset, traces with a cycle time lower than the mean duration of the traces in the event log (\textasciitilde\xspace28 days) have been labelled as positive, as negative otherwise. Similarly, in the \textsc{sepsis$_{median}$}, traces with a cycle time lower than the median duration of the traces in the event log (\textasciitilde\xspace 5 days)  have been labeled as positive; as negative otherwise.  

\textsc{bpic12}~\cite{BPIC2012} is a real-life event log pertaining to the application process for personal loans or overdrafts in a Dutch financial institute. It merges three intertwined sub-processes. Also in this case, the traces have been labelled based on their cycle execution time. In the \textsc{bpic12$_{mean}$} dataset (resp. \textsc{bpic12$_{mean}$}), traces with a cycle time lower than the mean (resp. median) duration of the traces in the event log (\textasciitilde\xspace 8 days, resp. \textasciitilde\xspace 19 hours) have been labelled as positive; as negative otherwise.

Table~\ref{tab:rl_datasets} summarizes the data related to the five resulting datasets.

\begin{table}[h]
	\renewcommand{\arraystretch}{1.2}
	\centering
	\scalebox{0.7}{
		\begin{tabular} {l l c c c c c}
		\toprule
			\multirow{2}{*}{\textbf{Dataset}} & \multirow{2}{*}{\textbf{Log}} & \multirow{2}{*}{\textbf{Trace \#}} & \multirow{2}{*}{\textbf{Activity \#}} & \multirow{2}{*}{\textbf{Label}} & \textbf{Positive} & \textbf{Negative}  \\ 
			& & & & & \textbf{Trace \#} & \textbf{Trace \#} \\ \midrule
			\textsc{cerv$_{compl}$} & \textsc{cerv} & 157 & 16 & compliant & 55 & 102 \\ \midrule
			\textsc{sepsis$_{mean}$} & \multirow{2}{*}{\textsc{sepsis}} & \multirow{2}{*}{1000} & \multirow{2}{*}{16} & mean duration & 838 & 212 \\
			\textsc{sepsis$_{median}$} &  &  &  & median duration & 525 & 525 \\ \midrule
			\textsc{bpic12$_{mean}$} & \multirow{2}{*}{\textsc{bpic12}} & \multirow{2}{*}{13087} & \multirow{2}{*}{36} & mean duration & 8160 & 4927  \\
			\textsc{bpic12$_{median}$} &  &  &  & median duration & 6544  & 6543 \\ 			
			\bottomrule
		\end{tabular}}
		\caption{Dataset description}
		\label{tab:rl_datasets}
\end{table}

The results obtained by applying the \nd algorithm are summarised in Table~\ref{tab:rl_results}. The table reports for each dataset, the results related to the \subsetclos (connected to the \emph{generality} criterion of Eq. \ref{eq:most-gen} and \ref{eq:redun}) and \minclos (\emph{simplicity} criterion of Eq. \ref{eq:simpl1} and \ref{eq:simpl2}) optimizations in terms of number of returned models\footnote{We stop generating models after $20$ models, i.e., \para{max} in Table~\ref{tab:rl_results} indicates that more than $20$ models have been returned.}, minimum size of the returned models, as well as percentage of negative traces violated by the returned model. Moreover, the table reports the time required for computing the set of compatibles, the set of choices, as well as the \subsetclos and \minclos optimizations.

\begin{table*} [h]
	\renewcommand{\arraystretch}{1.2}
	\centering
	\scalebox{0.7}{
		\begin{tabular} {l c c c  c c c c c c c}
		\toprule
			\multirow{4}{*}{\textbf{Dataset}} & \multicolumn{3}{c}{\textbf{\subsetclos}} &  			\multicolumn{3}{c}{\textbf{\minclos}}  & 	\multicolumn{4}{c}{\textbf{ Required Time (s)}} \\ 
			\cmidrule(r){2-4} \cmidrule(r){5-7} \cmidrule(r){8-11}
			& \textbf{Number of} & \textbf{Min model} & \textbf{Violated} & \textbf{Number of } & \textbf{Min model} & \textbf{Violated} & \multirow{2}{*}{\textbf{Comp.}} & \multirow{2}{*}{\textbf{Choices}} & \multirow{2}{*}{\textbf{\subsetclos}} & \multirow{2}{*}{\textbf{\minclos}}  \\			
			& \textbf{models} & \textbf{size} & \textbf{$L^{-}$ trace \%} & \textbf{models} & \textbf{size} & \textbf{$L^{-}$ trace \%} & & & & \\ \midrule
			 \textsc{cerv$_{compl}$} & \para{max} & 4 & 100\% & \para{max} & 4 & 100\% & 0.12 & 0.33 & 0.065 & 0.045\\ \midrule
			\textsc{sepsis$_{mean}$} & 1 & 8 & 4.25\% & 1 & 8  & 4.25\% & 0.73 & 1.04 & 0.039 & 0.035\\ 
			\textsc{sepsis$_{median}$} & \para{max} & 14 & 26.86\% & 16 & 14 & 26.86\% & 0.45 &	1.5 & 0.2 & 0.087 \\ \midrule
			\textsc{bpic12$_{mean}$} & \para{max} & 12 & 1.42\% & \para{max} & 12 & 1.42\% & 13.51 & 31 & 0.096 & 0.066  \\ 
			\textsc{bpic12$_{median}$} & \para{max} & 23 & 36.59\% & \para{max} & 22 & 36.59\% & 13.32 & 37.63 & 359.164 & 43.846 \\ 			
			\bottomrule
		\end{tabular}}
		\caption{\nd results on the real-life logs}
		\label{tab:rl_results}
\end{table*}

The table shows that for the \textsc{cerv$_{compl}$} dataset, \nd is able to return models that satisfy the whole set of positive traces and violate the whole set of negative traces (the percentage of violated traces in $L^-$ is equal to 100\%) with a very low number of constraints (4). For the other datasets, the returned models are always able to satisfy all traces in $L^+$, however not all the negative traces are violated by the returned models. In case of the datasets built by using the mean of the trace cycle time, the percentage of violated traces is relatively small ($4.25\%$ for \textsc{sepsis$_{mean}$} and $1.24\%$ for \textsc{bpic12$_{mean}$}), as the number of constraints of the returned models ($8$ for \textsc{sepsis$_{mean}$} and $12$ for \textsc{bpic12$_{mean}$}). Nevertheless, \nd is able to obtain reasonable results with the real life datasets built with the median of the trace cycle time. Indeed, it is able to identify $14$ (resp. $22$-$23$) constraints able to accept all traces in $L^+$ and to cover about $27\%$ (resp. $37\%$) of the traces in $L^-$ for \textsc{sepsis$_{median}$} (resp. \textsc{bpic12$_{median}$}). The difference in terms of results between the \textsc{cerv$_{compl}$} and the other datasets is not surprising. Indeed, while what characterizes positive and negative traces in the \textsc{cerv$_{compl}$} dataset depends upon the control flow (i.e., it depends on whether each execution complies with the cervical cancer screening protocol adopted by the screening center), when mean and median cycle time are used, the difference between positive and negative traces could likely not exclusively depend upon the control flow of the considered traces. Overall, the inability to identify a set of constraints that is able to fulfil all traces  in $L^+$ and to violate all negative ones is due to a bias of the considered language (Declare without data) that does not allow to explain the positive traces without the negative ones.

The difference of the results obtained with the mean and the median cycle time can also be explained as a language bias issue for the specific labelled datasets. Indeed, while when the positive and negative trace sets are quite balanced (i.e., for \textsc{sepsis$_{median}$} and \textsc{bpic12$_{median}$}) \nd is able to identify a set of constraints (related to the control flow) describing the traces with a low-medium cycle time and excluding the ones with a medium-high cycle time, when the sets of the positive and the negative traces are quite imbalanced (i.e., for \textsc{sepsis$_{mean}$} and \textsc{BPIC12$_{mean}$})
 characterizing the high number of traces with a low or medium cycle time while excluding the ones with a very high cycle time can become hard. 

The table also shows that \nd is overall very fast for small datasets (e.g., less than one minute for \textsc{cerv$_{compl}$}), while it requires some more time for large ones (e.g., \textsc{bpic12$_{mean}$} and \textsc{bpic12$_{median}$}). While the time required for computing \textit{compatibles} and \textit{choices} seems to be related to the size of the dataset, the time required for computing the optimizations seems to depend also on other characteristics of the datasets. 



Compared to state-of-the-art techniques for the discovery of declarative models starting from the only positive traces, \nd is able to return a small number of constraints satisfying all traces in $L^+$ without decreasing the percentage of violated traces in $L^-$.  
Among the classical declarative discovery approach, we selected the state-of-the-art \declareminer algorithm~\cite{2018a-Maggi} implemented in the \rum toolkit~\cite{2020-Alman}. We discovered the models using the only positive traces and setting the \para{support} parameter, which measures the percentage of (positive) traces satisfied by the \declare model, to 100\%\footnote{We run the \declareminer algorithm with vacuity detection disabled, \para{activity support filter} set to 0\%, using both transitive closure and hierarchy-based reduction of the discovered constraints, as well as with the whole set of Declare templates.}.

Table~\ref{tab:rl_declare_miner} summarizes the obtained results. The table reports for each dataset, the size of the model in terms of number of constraints, as well as the percentage of negative traces violated by the model. For lower values of the \para{support} parameter, i.e., for a lower percentage of positive traces satisfied by the model, the model returned by the \declareminer violates a higher percentage of negative traces. In this way, the \para{support} parameter allows for balancing the percentage of positive trace satisfied and negative traces violated. 

As hypothesised, the optimisation mechanism in \nd is able to identify a small set of constraints, that guarantees the satisfaction of all traces in $L^+$ and the same percentage of negative trace violations obtained with \declareminer (with \para{support} to 100\%).

\begin{table} [ht]
	\centering
	\scalebox{0.8}{
		\begin{tabular} {l c c}
		\toprule
			\textbf{Dataset} & \textbf{Model size} & \textbf{Violated L$^{-}$ trace \%}  \\ \midrule
			\textsc{cerv$_{compl}$} & 323 & 100\% \\ \midrule
			\textsc{sepsis$_{mean}$} & 210 & 4.25\%\\ 
			\textsc{sepsis$_{median}$} & 202 & 26.86\% \\ \midrule
      \textsc{bpic12$_{mean}$} & 514 & 1.42\% \\ 
			\textsc{bpic12$_{median}$} & 532 & 36.59\% \\ 
		\bottomrule
		\end{tabular}}
		\caption{\declareminer results}
		\label{tab:rl_declare_miner}
\end{table}

Finally, we evaluated the results obtained with \nd relying on the same procedure and dataset (\textsc{cerv$_{compl}$}) used in ~\cite{2007b-Lamma} to assess the results of \decminer, a state-of-the-art declarative discovery approach based on Inductive Logic Programming that is able to use both positive and negative execution traces. 
Five fold-cross validation is used, i.e., the \textsc{cerv$_{compl}$} dataset is divided into 5 folds and, in each experiment, 4 folds are used for training and the remaining one for validation purposes. The average \emph{accuracy} of the five executions is collected, where the accuracy is defined as the sum of the number of positive (compliant) traces that are (correctly) satisfied by the learned model and the number of negative (non-compliant) traces that are (correctly) violated by the learned model divided by the total number of traces. 
\begin{table} [b]
	\centering
	\scalebox{1}{
		\begin{tabular} {l c}
			\toprule
			\textbf{Approach} & \textbf{Accuracy} \\ \midrule
			 \decminer &  97.44\%\\ 
			 \declareminer & 96.79\%\\ 
			 \nd (\subsetclos) & 97.38\% \\ 			
			 \nd (\minclos) & 97.57\% \\ 
			\bottomrule
		\end{tabular}}
		\caption{Accuracy results obtained with \declareminer, \decminer and \nd}
		\label{tab:acc_results}
\end{table}

Table~\ref{tab:acc_results} reports the obtained accuracy values for the \decminer, the \declareminer (with the \para{support} parameter set to 100\%) and the \nd (both for the \minclos and \subsetclos optimization) approach. The table shows that on this specific dataset, \nd and \decminer have very close performance (with \nd \minclos performing slightly better). \declareminer presents instead on average a slightly lower accuracy mainly due to the highly constrained discovered model that, on one hand, allows for violating all negative traces in the validation set, and, on the other hand, leads to the violation of some of the positive traces in the validation set.


\section{Related work}
\label{sec:related}

Process discovery is generally considered a challenging task of process mining \cite{2012-Maggi}. The majority of works in this field are focused on discovering a business process model from a set of input traces that are supposed compliant with it. In this sense, process discovery can be seen as the application of a machine learning technique to extract a grammar from a set of positive sample data. Angluin et al. \cite{1983-Angliun} provide an interesting overview on this wide field.
Differently from grammar learning, where the model is often expressed with automata, regular expressions or production rules, process discovery usually adopts formalisms that can express concurrency and synchronization in a more understandable way \cite{2009-Goedertier}. 
The language to express the model is a crucial point, which inevitably influences the learning task itself. Indeed, the two macro-categories of business process discovery approaches---procedural and declarative---differ precisely by the type of language to express the model.
Well known examples of procedural process discoverer are the ones presented in the works \cite{2003-Weijters,2004-Aalst,2007-Gunther,2010-Aalst,2013-Leemans,2015-Guo,2017-Augusto, 2019-Augusto}. 
Like most procedural approaches to process discovery, all these works contemplate the presence of non-informative noise in the log, which should be separated from the rest of the log and disregarded.
 
Traditional declarative approaches to process discovery stem from the necessity of a more friendly language to express loosely-structured processes. Indeed---as also pointed out by \cite{2012-Maggi}---process models are sometimes less structured than one could expect. The application of a procedural discovery could produce spaghetti-models. In that case, a declarative approach is more suitable to briefly list all the required or prohibited behaviours in a business process.
Similarly to our technique, the one exposed by Maggi et al. in \cite{2011-Maggi} starts by considering the set of all activities in the log and building a set of all possible candidate Declare constraints. That work stems from the idea that Apriori-like approaches---such as sequence mining \cite{1994-Agrawal} and episode mining \cite{1997-Mannila}---can discover local patterns in a log, but not rules representing prohibited behaviours and choices. Therefore, differently from our algorithm, the candidate Declare constraints are then translated into the equivalent \ac{LTL} and checked (one at a time) against all the log content employing the technique of \cite{2005-Aalst}. The process continue until certain levels of recall and specificity are reached. The performance of this technique is improved in \cite{2012-Maggi} with an interesting two-step approach to both reduce the search space of candidate constraints and exclude from the model those \ac{LTL} formulas that are vacuously satisfied.
Also the work \cite{2012-Schunselaar} by Schunselaar et al. proposes a model refinement to efficiently exclude vacuously satisfied constraints. 
The MINERFul approach described in \cite{2015-DiCiccio} proposes to employ four metrics to guide the declarative discovery approach: support, confidence and interest factor for each constraint w.r.t. the log, and the possibility to include in the search space constraints on prohibited behaviours.
Particularly relevant for our purposes is the work by Di Ciccio et al. \cite{2017-DiCiccio}, who focus on refining Declare models to remove the frequent redundancies and inconsistencies. The algorithms and the hierarchy of constraints described in that work were particularly inspiring to define our discovery procedure.
Similarly to the procedural approaches, all the declarative ones described so far do not deal with negative example, although the vast majority of them envisage the possibility to discard a portion of the log by setting thresholds on the value of specific metrics that the discovered model should satisfy.
 
In the wider field of grammar learning, the foundational work by Gold \cite{1967-Gold} showed how negative examples are crucial to distinguish the right hypothesis among an infinite number of grammars that fit the positive examples. Both positive and negative examples are required to discover a grammar with perfect accuracy. Since process discovery does not usually seek perfection, but only a good performance according to defined metrics, it is not surprising that many procedural and declarative discoverers disregard the negative examples. Nonetheless, in this work we instead claim that negative traces are extremely important when learning declarative process models.
 
The information contained in the negative examples is actively used in a subset of the declarative process discovery approaches \cite{2007-Lamma,2009-Chesani,2010-Bellodi,2016-Bellodi}. All these works can be connected to the basic principles of the \ac{ICL} algorithm \cite{1995-DaRaedt}, whose functioning principle is intrinsically related to the availability of both negative and positive examples. 
The \ac{DPML} described in \cite{2007-Lamma, 2007b-Lamma} by Lamma et al. focuses on learning integrity constraints expressed as logical formulas. The constraints are later translated into an equivalent construct of the declarative graphical language DecSerFlow \cite{2006-Aalst}. Similarly to this approach, the DecMiner tool described in \cite{2009-Chesani}, learns a set of SCIFF rules \cite{2008-Alberti} which correctly classify an input set of labelled examples. Such rules are then translated into ConDec constraints \cite{2006-Pesic}. Differently from \cite{2009-Chesani}, the present work directly learns Declare constraints without any intermediate language.
\ac{DPML} has been later used in \cite{2010-Bellodi} to extract integrity constraints, then converted into Markov Logic formulas. The weight of each formula is determined with a statistical relational learning tool. 
Taking advantage of the negative examples, the approach of \cite{2010-Bellodi}  is improved in \cite{2016-Bellodi}, thus obtaining significantly better results than other process discovery techniques. 
Since for all these works the availability of negative examples is crucial, recent years have seen the development of synthetical log generators able to produce not only positive but also negative process cases \cite{2019-Chesani,2017-Chesani,2020-Loreti,2009-Goedertier, 2014-Stocker, 2010-Hee}. In the experimental evaluation of this work, we employ the abductive-logic generator by Loreti et al. \cite{2020-Loreti} to synthesise part of the input event logs.
 
Particularly related to our approach are the works by Neider et al. \cite{2018-Neider}, Camacho et al. \cite{2019-Camacho}, and Reiner \cite{2019-Riener} where a SAT-based solver is employed to learn a simple set of LTL formulas consistent with an input data set of positive and negative examples. In particular, Neider et al. \cite{2018-Neider} employ decision tree to improve the performance and manage large example sets; Camacho et al. \cite{2019-Camacho} exploit the correspondence of LTL formulae with Alternating Finite Automata (AFA); whereas Reiner uses partial Directed Acyclic Graphs (DAGs) to decompose the search space into smaller subproblems.
The concept of negative example used in this work could be related to both the definitions of syntactical and semantic noise of \cite{2009-Gunther}. In particular, besides being able to extract relevant syntactic information that characterise the positive examples w.r.t. negative, our approach could also be useful to deal with the semantic concept of modification noise i.e., the semantic difference between traces from the same process model, which has been partially or totally modified at a certain point in time.  As a minor point, we also might notice that these works provides in output LTL formulas, while we opt for Declare formulas with LTL$_f$ semantics.
 
It is important to underline that also a limited number of procedural approaches envisage the need for taking into account the information contained into the negative examples. 
In particular, the work \cite{2015-Ponce} showed how negative examples can be employed to alleviate problems like log incompleteness and noise.
The AGNEs tool described in \cite{2009-Goedertier} increases the dimension of an event log with artificially generated negative examples, then uses \ac{ILP} multi-relational classification to discover a Petri net model. 
Negative examples are generated in a rather syntactical way, by adding a unique negative event at the end of a positive trace. This concept is extended in the work of Ponce de Leon et al. \cite{2018-Ponce} which envisage the concatenations of a sequence of negative events. Differently from these approaches, our technique does not assume syntactical restrictions on the input negative examples.

\ac{ILP} is also used in \cite{2006-Ferreira}, where the authors suppose a set of negative examples provided by domain experts. The approach uses partial-order planning to discover a structured model. More recently, the works \cite{2014-Broucke,2014-BrouckePhD} showed how synthetically generated traces can be employed to improve the robustness of the compliance monitor task. 
 
Deviant cases---intended as traces whose sequence of activities deviates from the expected behaviour---are the subject of deviance mining approaches reviewed and evaluated by Nguyen et al. in \cite{2016-Nguyen}. Some applications of deviance mining tend to highlight the differences between models discovered from deviant and non-deviant traces \cite{2014-Suriadi,2014-Armas}. Other works intend deviance mining as a classification task, where the miner is required to identify normal and deviant traces given a set of examples. The classification inherently causes the discovery of patterns which distinguish different types of traces. In this sense, deviance mining is particularly similar to sequence classification. The discovered patterns can be based on the simple frequency of individual activities as in \cite{2013-Suriadi,2015-Partington}, their co-occurrence as in \cite{2011-Swinnen}, or the occurrence of specific subsequences \cite{2013-Bose,2007-Lo,2016-Bernardi}.

%
%
%


\section{Conclusion}
While the vast majority of works see process discovery as a one-class supervised learning task, we embrace the less popular view of process discovery as a binary supervised learning job, where traces representing a ``stranger'' behaviour can be considered as heralds of valuable information about the process itself. Devoted to this vision, we developed a technique which considers both positive and negative traces and performs process discovery as a satisfiability problem, where different heuristics can be adopted to assess the optimal model according to different goals.  

Being able to extract valuable knowledge from negative examples, our proposal \nd can be employed to enrich the process description extracted by a state-of-the-art declarative process discovery algorithm.
However, it is worth to underline that the resulting declarative process discoverer taking advantage of explicitly defined positive and negative examples would not be necessarily an alternative to procedural discovery techniques. 
Indeed in some cases, when correct thresholds and language biases are adopted, procedural discoverers have the great advantage to provide the user with a rather easy-to-understand definition of the process model. Nonetheless, the informative content provided by those process cases that are discarded by procedural discoverer can still be extremely important. 
For the future, as \nd extracts valuable information from $L^-$ without excluding any trace of $L^+$, it could also be applied as a post processing technique to enrich the output of procedural discoverers with declarative constraints.
The resulting output would be an hybrid procedural/declarative process model, showing a simple and handy structured representation of the main business process together with a set of declarative constraints. The goal of such constraints would be to account for less frequent deviances or prohibited behaviours in a much more synthetic and easy-to-understand way with respect to an equivalent spaghetti-like procedural formulation.

Furthermore, such a hybrid solution could also greatly simplifying the elicitation of long-term dependencies between activities that occur at the beginning of the process and those carried out towards the end. Indeed, the structured nature of procedural approaches makes them not properly suitable to express such dependencies.
One current way to tackle such issue is through the employment of global variables and ``if'' statements to control the execution flow of each instance. For example, Kalenkova et al. \cite{2020-Kalenkova} propose a process discovery technique devoted to repair free-choice procedural workflows with additional modelling constructs, which can more easily capture non-local dependencies. Nonetheless, since such additional constraints are intended to preserve the procedural nature of the model, the result may increase its complexity and ultimately affect its readability.
An hybrid procedural/declarative model formulation would maintain a structured form to express the model while integrating it with handy declarative long-term constraints involving activities occurring far from each other in the workflow.
%
%
This idea of a hybrid procedural/declarative model formulation has been explored by various works and proved to be particularly effective in the field of medical clinical guidelines \cite{2009a-Bottrighi,2009b-Bottrighi, 2011-Bottrighi}. A wider landscape of applications is considered by Maggi et al. in the work \cite{2018b-Maggi}. 

Finally, the performance of \nd presented in this paper could be boosted through a parallel approach. Analogously to previous works \cite{2018a-Maggi, 2018-Loreti, 2020b-Loreti} we can envisage two possible directions to decompose our task: by spitting the model (i.e. in this case the set of constraints to be learned), or the input data (i.e. the business log). The algorithm presented here could easily adopt the first kind of partitioning, whereas the second might be more challenging.

\section*{Acknowledgment}
This work was partially supported by the European Commission funded project ``Humane AI: Toward AI Systems That Augment and Empower Humans by Understanding Us, our Society and the World Around Us'' (grant \# 820437). The support is gratefully acknowledged.


\bibliographystyle{alpha}

\newcommand{\etalchar}[1]{$^{#1}$}

\end{document}